\documentclass[twoside,11pt]{article}

\usepackage[preprint]{jmlr2e}

\usepackage[T1]{fontenc}
\usepackage{lmodern}
\usepackage{microtype}
\usepackage{amsmath,mathtools,bm}
\usepackage{booktabs,multirow,array}
\usepackage{xcolor}
\usepackage{tikz}
\usetikzlibrary{arrows.meta,positioning,shapes.geometric,fit,calc}
\usepackage[capitalise,noabbrev]{cleveref}
\usepackage{enumitem}
\usepackage{caption}
\usepackage{lastpage}
\hypersetup{colorlinks=true,linkcolor=blue!55!black,citecolor=blue!55!black,urlcolor=blue!55!black}

\makeatletter
\@for\thm:=proposition,corollary,lemma,definition,remark\do{%
  \expandafter\let\csname\thm\endcsname\relax\expandafter\let\csname end\thm\endcsname\relax}
\makeatother
\newtheorem{proposition}{Proposition}
\newtheorem{corollary}{Corollary}

\newtheorem{definition}{Definition}
\newtheorem{assumption}{Assumption}
\newtheorem{remark}{Remark}
\renewenvironment{proof}[1][Proof]{\par\noindent{\bf #1\ }}{\hfill\BlackBox\\[2mm]}

\newcommand{\E}{\mathbb{E}}
\newcommand{\R}{\mathbb{R}}
\newcommand{\rms}{\operatorname{rms}}
\newcommand{\etac}{\eta_{c}}
\newcommand{\bc}{b_{c}}
\newcommand{\aptq}{\alpha_{\mathrm{PTQ}}}
\newcommand{\aqat}{\alpha_{\mathrm{QAT}}}
\newcommand{\norm}[1]{\left\lVert #1\right\rVert}

\ShortHeadings{Depth Laws for the Precision Floor}{Tarawneh}
\firstpageno{1}

\begin{document}

\title{Depth Laws for the Precision Floor of Trained Neural Networks:
Amplification, Residual Scaling, and a Quantization-Aware Training Paradox}

\author{\name Ahmad S. Tarawneh \email Ahmad.trwh@mutah.edu.jo \\
       \addr Faculty of Information Technology\\
       Mutah University\\
       Al-Karak, Jordan}

\maketitle

\begin{abstract}
How many bits does a network need before its accuracy collapses, and how does this grow with depth? We
study the \emph{precision floor}, the perturbation level or bit-width at which accuracy falls halfway to
chance, in MLPs, CNNs, Vision Transformers and nine pretrained language models, under post-training
quantization (PTQ) and quantization- or noise-aware training (QAT).
(i) A first-order theory sets the floor through one full-precision quantity, the \emph{predictive
amplification} $G$: $\etac=\Lambda/G$, and $G^2$ grows linearly in depth at a rate proportional to the
squared residual branch scale.
(ii) The predicted equality $\aptq=\rho$ of depth exponents holds within 95\% intervals in twelve of
thirteen trained architectures and in GPT-2 from 12 to 48 layers, with $\Lambda=1.45\pm14\%$ across
trained architectures.
(iii) Residual branches scaled by $1/\sqrt{D}$ and pre-normalisation remove the depth penalty, and each
quantizer turns noise into bits at a rate fixed by its step rule, giving $\bc=(\alpha/\gamma)\log_2D+C$.
(iv) A \emph{QAT paradox}: noise-aware training roughly doubles the tolerable noise of shallow networks,
but the gain decays with depth, so the depth law steepens ($\aqat/\aptq=1.45$--$1.47$ on two datasets,
ten seeds each).
Decision margins, cross-layer error cancellation and heavy tails do not set the floor.
\end{abstract}


\section{Introduction}\label{sec:intro}

Low-precision inference is now the default way to deploy neural networks, and a large literature develops
quantizers, calibration schemes and quantization-aware training (QAT) recipes that push weights and
activations to four bits or fewer \citep{jacob2018quantization,esser2020lsq,nagel2021white,gholami2021survey}.
Much less is known about a simpler structural question: \emph{how does the precision a network needs depend
on its architecture, and in particular on its depth?} Every layer of a quantized network injects error, and
every later layer can amplify or attenuate it. Depth should therefore matter---but in which direction, by how
much, and does training with quantization change the answer?

This paper gives quantitative answers. We define the \emph{precision floor} of a network as the perturbation
level $\etac$ (or bit-width $\bc$) at which test accuracy falls halfway from its full-precision value to
chance, and we measure how it scales with depth $D$ across four families: MLPs, CNNs, Vision Transformers
trained by us, and nine publicly released language models that we only evaluate. Across all of them the
floor follows a power law in depth,
\begin{equation}
  \etac(D) \;\propto\; D^{-\alpha},
  \qquad
  \bc(D) \;=\; \frac{\alpha}{\gamma_q}\,\log_2 D + C_q ,
  \label{eq:headline}
\end{equation}
whose exponent $\alpha$ is set by the architecture and whose conversion into bits is set by a single
quantizer constant $\gamma_q$. The organising idea is a two-link chain (\cref{fig:concept}): the architecture
determines how perturbations are \emph{amplified} on their way to the output distribution; amplification
determines the \emph{noise} tolerance; and the quantizer determines how noise tolerance translates into
\emph{bits}. The practical upshot is that the floor can be predicted before quantizing anything: across all
architectures we trained, $\etac\approx1.45/G$ to within about 15\%, where $G$ is measured at full precision
in seconds.

\begin{figure}[t]
\centering
\begin{tikzpicture}[
  node distance=9mm and 11mm,
  box/.style={draw, rounded corners=3pt, align=center, font=\small, minimum height=15mm, text width=27mm, line width=0.6pt},
  arch/.style={box, fill=gray!10},
  amp/.style={box, fill=blue!8, draw=blue!60!black},
  noise/.style={box, fill=teal!8, draw=teal!60!black},
  bits/.style={box, fill=orange!10, draw=orange!70!black},
  side/.style={box, fill=white, text width=30mm, minimum height=12mm},
  arrS/.style={-{Stealth[length=2.6mm]}, line width=0.9pt},
  arrD/.style={-{Stealth[length=2.6mm]}, line width=0.9pt, dashed},
  lab/.style={font=\scriptsize, midway, fill=white, inner sep=1.5pt}
]
\node[arch]  (A) {\textbf{Architecture}\\[1pt] depth $D$, branch\\ scale $s$};
\node[amp, right=of A] (G) {\textbf{Amplification}\\[1pt] $G^2 = G_0^2 + \kappa s^2 D$\\[1pt] {\scriptsize measured at full precision}};
\node[noise, right=of G] (N) {\textbf{Noise floor}\\[1pt] $\eta_c = \Lambda / G$\\[1pt] $\eta_c \propto D^{-\alpha}$};
\node[bits, right=of N, text width=33mm] (B) {\textbf{Bit floor}\\[2pt] $b_c=\tfrac{\alpha}{\gamma_q}\log_2 D + C_q$};
\node[side, below=10mm of N] (Q) {\textbf{QAT} (noise-aware)\\ $\eta_c^{\mathrm{QAT}} = g(D)\,\eta_c$,\ $g\!\downarrow$ with $D$};
\node[side, below=10mm of B] (U) {\textbf{Quantizer} $q$\\ $\eta_q(b)\propto 2^{-\gamma_q b}$};
\draw[arrS] (A) -- node[lab, above=1pt] {Prop.~\ref{prop:growth}} (G);
\draw[arrS] (G) -- node[lab, above=1pt] {Prop.~\ref{prop:collapse}} (N);
\draw[arrD]  (N) -- node[lab, above=1pt] {Cor.~\ref{cor:bitfloor}} (B);
\draw[arrD]  (N) -- node[lab, right=1pt] {Sec.~\ref{sec:qat}} (Q);
\draw[arrS] (U) -- node[lab, right=1pt] {Prop.~\ref{prop:gamma}} (B);
\node[font=\scriptsize, anchor=north west] at ($(A.south west)+(0,-6mm)$) {%
  \begin{tabular}{@{}l@{\ }l@{}}
  \tikz\draw[arrS] (0,0)--(6mm,0); & quantitatively confirmed\\[1pt]
  \tikz\draw[arrD]  (0,0)--(6mm,0); & approximate / conditional
  \end{tabular}};
\end{tikzpicture}
\caption{\textbf{The two-link structure of the precision floor.} Solid arrows are relations supported
quantitatively by our experiments; dashed arrows are approximate or conditional. The architecture fixes
the amplification $G(D)$ (\cref{prop:growth}); in PTQ the noise floor is $\etac\propto 1/G$
(\cref{prop:collapse}); the quantizer converts noise into bits at rate $\gamma_q$ (\cref{prop:gamma}).
Noise-aware training (QAT) raises $\etac$ but with a steeper depth exponent (\cref{sec:qat}).}
\label{fig:concept}
\end{figure}
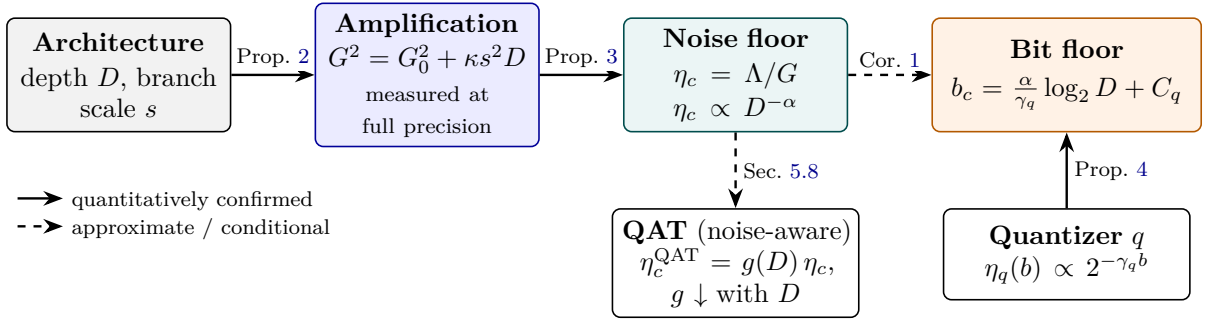

\paragraph{Contributions.}
\begin{enumerate}[leftmargin=*,itemsep=2pt]
  \item \textbf{A first-order theory of the floor.} For small perturbations the change of the output
  distribution is governed by one full-precision quantity, the predictive amplification $G$; its square is a
  sum of per-site sensitivities and grows approximately linearly in depth with slope proportional to $s^2$, the squared
  residual branch scale; and under a collapse functional the noise floor is $\etac=\Lambda/G$
  (\cref{sec:framework}).
  \item \textbf{Amplification predicts the depth exponent, in four families.} The exponents of the floor and
  of the amplification agree within their 95\% intervals in twelve of thirteen trained architectures (MLPs,
  CNNs with and without normalisation, ViTs) and in GPT-2 from 12 to 48 layers, over exponents from $-0.57$
  to $+0.32$. The collapse constant $\Lambda=\etac G$ is $1.45\pm14\%$ across all trained architectures
  and $1.01$--$1.13$ across GPT-2 (\cref{sec:predictor,sec:vit,sec:llm}).
  \item \textbf{A validity condition and a measurement rule.} The law is first-order, and its premise can be
  checked on any model: $G$ must not depend on the noise level used to measure it. Two small pretrained
  Pythia models fail this check, and the law correctly makes no claim about them. For large-vocabulary
  models, amplification must be measured on the output distribution rather than the logits (\cref{sec:llm}).
  \item \textbf{Architecture controls the penalty.} Residual branches scaled by $1/\sqrt D$ and
  pre-normalisation make the floor depth-independent in MLPs, CNNs and ViTs; a fixed-scale control restores
  the penalty, isolating the scaling rather than the skip connection as the cause. In a depth-scaled ViT,
  deeper networks need fewer bits (\cref{sec:residual,sec:norm,sec:vit}).
  \item \textbf{A calibrated noise-to-bits bridge.} A quantizer's bits-to-noise rate follows from its step
  rule ($\gamma=\tfrac12$ for LSQ-style, $1$ for min--max scaling), matching measurement in every family, and
  turns noise laws into bit laws (\cref{sec:bits}).
  \item \textbf{The QAT paradox.} Noise-aware training roughly doubles the tolerable noise of shallow
  networks, yet the benefit shrinks with depth, so the QAT depth law is \emph{steeper} than the PTQ law:
  $\aqat/\aptq\approx1.5$ on two datasets with ten seeds each. The same decay appears in a CNN, where about
  half of it is a training-budget effect, and in a Vision Transformer. A per-layer sensitivity map shows that QAT reduces amplification uniformly and by
  a depth-independent factor, so the decay originates in the nonlinear collapse regime
  (\cref{sec:qat,sec:mechanism}).
\end{enumerate}
We also report hypotheses that did \emph{not} survive testing---decision margins, cross-layer error
cancellation and activation heavy tails do not set the floor (\cref{sec:negative})---and a width study which
shows that one of our own early claims, a parameter-free value of the growth constant, was a coincidence of
network width (\cref{app:width}). Predictions reported as tests were fixed before the corresponding experiment.

\section{Related work}\label{sec:related}

\paragraph{Sensitivity and required precision.}
The idea that a network's sensitivity to perturbations determines the precision it needs goes back at
least to \citet{sakr2017analytical}, who bound the probability that fixed-point quantization changes a
decision using noise gains computed by back-propagation, and use the bound to choose per-network
precision. Hessian- and sensitivity-based mixed-precision methods \citep{dong2019hawq} follow the same
principle to allocate bits across layers. Our amplification $G$ is a Monte-Carlo analogue of these noise
gains. What is new here is not the principle but its \emph{depth-scaling} form: we derive and measure how
$G$, and with it the floor, scales with depth across architectures, and we ask how training with noise
changes that scaling.

\paragraph{Depth, quantization and signal propagation.}
\citet{blumenfeld2019mean} analyse quantized activations with mean-field theory and derive a
quantization--depth trade-off for trainability at initialization, $L_{\max}\propto N^{1.82}$ for $N$
quantization levels. We study a different object---the accuracy-defined floor of \emph{trained} networks,
separately for PTQ and QAT---but the two results share the logic that depth multiplies the effect of
per-layer error. Our growth law (\cref{prop:growth}) rests on the same norm-propagation reasoning as
signal-propagation theory \citep{poole2016exponential,schoenholz2017deep}. The special role of
$1/\sqrt{D}$ branch scaling in residual networks is well established for stability and for well-defined
infinite-depth limits \citep{hayou2023width,bordelon2024depthwise,yang2024tensor}; we show its
consequence for the precision floor, and isolate it with a fixed-scale control. Recently,
\citet{soltanalian2026depth} developed a resource theory of when depth can compensate for precision in
quantized residual computation, and \citet{ou2024three} show that high-precision ReLU networks can be converted into deeper low-precision ones; both focus on expressivity rather than on the floor of trained classifiers.

\paragraph{Error accumulation in post-training quantization.}
\citet{chen2026ptq} study how quantization error accumulates across the layers of pretrained language
models and identify \emph{counteraction}---new error partially cancelling inherited error---as a reason
PTQ works well in practice, and \citet{arai2025qep} make cross-layer error propagation explicit in layer-wise PTQ. We test this mechanism in our networks and do not find it
(\cref{sec:negative}). \citet{residualfree2026} argue that residual connections make transformer
activations heavy-tailed and hence harder to quantize; we observe heavy tails growing with depth in all
our residual CNNs, but find that they do not by themselves make quantization disproportionately harmful.

\paragraph{Training with noise.}
Training with injected noise improves robustness to that noise; it is equivalent, to first order, to a
Tikhonov-type regulariser \citep{bishop1995training}. \citet{zhou2020noisy} show that deeper, narrower
networks lose information faster under weight noise and that noise-aware training and distillation
recover much of it; Quant-Noise \citep{fan2020training} trains with quantization noise on subsets of
weights. Relative Gaussian activation noise is mathematically Gaussian dropout
\citep{srivastava2014dropout}, whose implicit regularisation in deep linear networks has been
characterised \citep{mianjy2018implicit}. None of these works reports how noise-aware training changes the
\emph{depth exponent} of robustness; our finding that it steepens it (\cref{sec:qat}) is, to our
knowledge, new. At the scale of language models, ParetoQ \citep{liu2025paretoq} reports a pronounced
learning transition between two and three bits in QAT, and \citet{catalan2026training} show that training dynamics, in particular learning-rate decay, change how robust a model is to PTQ---consistent with the training-budget dependence we observe in \cref{sec:qat}.

\paragraph{Scaling laws.}
Our laws are empirical power laws with architecture-dependent exponents, in the tradition of neural
scaling laws \citep{kaplan2020scaling}; unlike most such laws, the exponent here is predicted by an
independent full-precision measurement.

\section{A first-order theory of the precision floor}\label{sec:framework}

\subsection{Perturbation model and the precision floor}

Let $f_\theta:\R^{d}\to\R^{K}$ be a trained network with logits $z = f_\theta(x)$, composed of $D$ layers.
A \emph{perturbation site} is any tensor that a quantizer would act on: each weight matrix $W_l$ and
each activation $a_l$ entering a matrix multiplication or convolution. We write $\mathcal S$ for the
set of sites and $x_i$ for the tensor at site $i\in\mathcal S$.

\begin{definition}[Relative Gaussian perturbation]\label{def:perturb}
For a noise level $\eta\ge0$, every site is replaced by
\begin{equation}
  \tilde x_i \;=\; x_i + \eta\,\rms(x_i)\,\xi_i ,\qquad \xi_i\sim\mathcal N(0,I),\ \text{independent across sites},
  \label{eq:perturb}
\end{equation}
where $\rms(x)=(\tfrac1n\sum_j x_j^2)^{1/2}$ over the entries of the tensor. We write
$\tilde z(x;\eta)$ for the perturbed logits and $A(\eta)$ for test accuracy under this perturbation.
\end{definition}

The perturbation is \emph{relative}: it is scale-free, so that $\eta$ has the same meaning in every
layer and every architecture, and it is the natural continuous surrogate for quantization, whose error
is proportional to the quantization step and hence to the scale of the tensor (\cref{sec:bitsbridge}).

\begin{definition}[Precision floor]\label{def:floor}
Let $A_0 = A(0)$ be full-precision accuracy and $1/K$ chance. The \emph{noise floor} $\etac$ is the
noise level at which accuracy falls halfway to chance, $A(\etac) = \tfrac12(A_0 + 1/K)$. For a
quantizer $q$ with bit-width $b$, the \emph{bit floor} $\bc$ is defined identically with $A(\eta)$
replaced by the accuracy of the $b$-bit quantized network. The \emph{depth exponent} $\alpha$ is defined
by $\etac(D)\propto D^{-\alpha}$; we write $\aptq$ when the network was trained at full precision and
perturbed only at test time, and $\aqat$ when the same perturbation was also present during training.
\end{definition}

\subsection{Amplification}

For small $\eta$, a first-order expansion of the network around the unperturbed forward pass gives
\begin{equation}
  \tilde z - z \;=\; \eta \sum_{i\in\mathcal S} \rms(x_i)\, J_i\,\xi_i \;+\; O(\eta^2),
  \qquad J_i = \frac{\partial z}{\partial x_i}.
  \label{eq:firstorder}
\end{equation}

\begin{definition}[Amplification]\label{def:G}
The \emph{amplification} of a network on a data distribution is
\begin{equation}
  G \;=\; \lim_{\eta\to0}\;\frac{1}{\eta}\,
  \frac{\big(\E_{x}\E_{\xi}\norm{\tilde z-z}^2\big)^{1/2}}{\big(\E_x\norm{z}^2\big)^{1/2}} ,
  \label{eq:Gdef}
\end{equation}
the relative logit perturbation per unit relative input perturbation. The \emph{amplification exponent}
$\rho$ is defined by $G(D)\propto D^{\rho}$.
\end{definition}

\begin{proposition}[Amplification as a sum over sites]\label{prop:sum}
Under \cref{def:perturb},
\begin{equation}
  G^2 \;=\; \sum_{i\in\mathcal S} q_i,
  \qquad
  q_i \;=\; \frac{\rms(x_i)^2\,\E_x\norm{J_i}_F^2}{\E_x\norm{z}^2}.
  \label{eq:Gsum}
\end{equation}
\end{proposition}
\begin{proof}
Square \eqref{eq:firstorder}, take the expectation over the independent $\xi_i$ (cross terms vanish
because $\E[\xi_i\xi_j^\top]=0$ for $i\ne j$ and $\E[\xi_i\xi_i^\top]=I$), use
$\E_\xi\norm{J_i\xi_i}^2=\norm{J_i}_F^2$, divide by $\eta^2\E_x\norm z^2$ and let $\eta\to0$.
\end{proof}

\paragraph{Measuring amplification in prediction-relevant directions.}
The logit norm in \eqref{eq:Gdef} counts every direction of logit space equally, but classification does
not: adding the same constant to all logits changes nothing after the softmax, and perturbations of
classes with negligible probability barely change the prediction. We therefore use the
\emph{predictive} amplification
\begin{equation}
  G \;=\; \lim_{\eta\to0}\;\frac{1}{\eta}\Big(2\,\E_x\E_\xi\,\mathrm{KL}\big(p(\cdot\mid x)\,\big\|\,\tilde p(\cdot\mid x)\big)\Big)^{1/2},
  \label{eq:Gkl}
\end{equation}
where $p=\mathrm{softmax}(z)$ and $\tilde p=\mathrm{softmax}(\tilde z)$. To second order in $\eta$,
$2\,\mathrm{KL}(p\,\|\,\tilde p) = (\tilde z-z)^{\top}F(\tilde z-z)$ with the softmax Fisher matrix
$F=\mathrm{diag}(p)-pp^{\top}$, which annihilates common shifts ($F\mathbf 1=0$) and weights each class by
its probability. \eqref{eq:Gkl} is thus the logit-perturbation norm of \eqref{eq:Gdef} measured in the
metric that the prediction itself uses. All results below use \eqref{eq:Gkl}. For ten-class models the two
definitions give similar exponents (differences of at most $0.1$), with \eqref{eq:Gkl} tracking the
collapse exponent more closely; for language models with $5\times10^4$ classes only \eqref{eq:Gkl} is
meaningful (\cref{sec:llm}). \Cref{prop:sum} holds for \eqref{eq:Gkl} with $\norm{J_i}_F^2$ replaced by
$\operatorname{tr}(J_i^{\top}FJ_i)$, and the propositions below are unchanged.

Each site contributes a non-negative \emph{sensitivity} $q_i$: the fraction of the logit energy that a
unit relative perturbation at site $i$ would produce. Depth enters only through the number of sites and
through how $q_i$ depends on the site's position. This yields a growth law under a propagation
assumption that is standard in signal-propagation analyses \citep{poole2016exponential,schoenholz2017deep}.

\begin{assumption}[Norm-preserving propagation]\label{ass:iso}
In the hidden layers, a relative perturbation of size $\epsilon$ introduced into the signal path produces,
to first order, a relative logit perturbation of size $c\,\epsilon$ with $c$ independent of the layer
index; weight and activation sites contribute comparably.
\end{assumption}

\cref{ass:iso} holds at initialization for ReLU networks with variance-preserving scaling, where the
input--output Jacobian preserves norms in expectation. Whether it holds after training is an empirical
question; we test it directly in \cref{sec:amplification}.

\begin{proposition}[Growth law for amplification]\label{prop:growth}
Consider a network with $D$ hidden layers, $k$ perturbation sites per layer, and residual updates
$h_{l+1} = h_l + s\,F_l(h_l)$ (a plain network corresponds to replacing the stream by the branch, i.e.\
$s=1$ without the identity path). Under \cref{ass:iso},
\begin{equation}
  G(D)^2 \;=\; G_0^2 + \kappa\, s^2 D + o(D),
  \qquad \kappa = k\,c^2,
  \label{eq:growth}
\end{equation}
where $G_0^2$ collects the depth-independent sites (stem and head). Consequently:
\begin{enumerate}[label=(\alph*),itemsep=1pt]
  \item \emph{plain networks:} $G\propto D^{1/2}$ asymptotically, i.e.\ $\rho\to\tfrac12$;
  \item \emph{residual networks with $s = s_0/\sqrt D$:} $G^2\to G_0^2+\kappa s_0^2$, i.e.\ $\rho = 0$;
  \item \emph{residual networks with fixed $s$:} the local exponent
  $\rho(D) = \tfrac{d\log G}{d\log D} = \tfrac{\kappa s^2 D}{2(G_0^2+\kappa s^2 D)}$
  increases monotonically from $0$ towards $\tfrac12$.
\end{enumerate}
\end{proposition}
\begin{proof}[Proof sketch]
By \cref{prop:sum}, $G^2$ is a sum of site sensitivities. In a residual block, a unit relative
perturbation at a branch site perturbs the branch output by a relative amount of order one, which enters
the stream multiplied by $s$; relative to the stream it therefore has size $O(s)$, and by \cref{ass:iso}
it reaches the logits with relative size $c\,s$, contributing $c^2s^2$ to $G^2$. Summing $k$ sites over
$D$ blocks gives $\kappa s^2D$. Parts (a)--(c) follow by substitution and differentiation.
A full derivation, including the normalisation of the stream, is given in \cref{app:growth}.
\end{proof}

\begin{remark}\label{rem:kappa}
\Cref{prop:growth} fixes the \emph{form} of the depth dependence, not its constant. The coefficient
$\kappa=kc^2$ depends on how much of a perturbation reaches the prediction-relevant directions of the
output, and therefore on architecture and width. For width-32 MLPs we measure slopes per unit $s^2$ of
$18.8$ [16.4, 21.3] (plain) and $8.8$ [6.0, 11.6] (fixed-scale residual), and in the logit metric the slope
of plain MLPs falls by a factor of eight between widths 32 and 128 (\cref{app:width}). What the precision
floor inherits is the \emph{exponent}, and it is width-stable: $\rho = 0.51$ and $0.47$ for plain MLPs at
widths 32 and 128, against the asymptotic value $\tfrac12$ of \cref{prop:growth}(a).
\end{remark}

\subsection{From amplification to the noise floor}\label{sec:collapse}

Amplification describes the size of the logit perturbation; the floor also depends on how much
perturbation classification can absorb. Let $\mu(x) = (z_{(1)}-z_{(2)})/\norm{z}$ denote the
\emph{normalised margin} (gap between the two largest logits relative to the logit norm).

\begin{proposition}[Collapse functional]\label{prop:collapse}
Suppose that, conditionally on $x$, the logit perturbation is approximately isotropic Gaussian with
standard deviation proportional to $\eta\,G\,\norm{z(x)}$, as implied by \eqref{eq:firstorder} when many
sites contribute. Then test accuracy depends on $\eta$ and $G$ only through their product,
\begin{equation}
  A(\eta) = \bar A(\eta G), \qquad
  \bar A(u) = \E_x\!\left[\psi\!\left(\frac{\mu(x)}{u}\right)\right],
  \label{eq:collapse}
\end{equation}
for a fixed increasing function $\psi$. Hence
\begin{equation}
  \etac \;=\; \frac{\Lambda}{G},\qquad
  \aptq \;=\; \rho - \frac{d\log\Lambda}{d\log D},
  \label{eq:etac}
\end{equation}
where $\Lambda$ solves $\bar A(\Lambda)=\tfrac12(A_0+1/K)$ and depends only on the distribution of
normalised margins. If that distribution does not change systematically with depth, $\aptq = \rho$.
\end{proposition}
\begin{proof}
Under the stated approximation, the probability that the top logit is preserved depends only on the ratio
of the logit gap to the perturbation scale, $\mu(x)\norm{z(x)}/(\eta G\norm{z(x)}) = \mu(x)/(\eta G)$.
Averaging over $x$ gives \eqref{eq:collapse}; solving the midpoint condition gives \eqref{eq:etac}.
\end{proof}

\Cref{prop:collapse} is the source of our main prediction: \emph{the depth exponent of the precision floor
can be read off a full-precision measurement of amplification}. Two caveats are testable and we test
both. First, it is a first-order statement, while collapse occurs at finite $\eta$ (typically
$0.1$--$0.8$). Second, it assumes a depth-stable margin distribution; we show in \cref{sec:negative}
that the floor is in fact much less sensitive to margins than the idealised $\psi$ suggests, so we treat
$\Lambda$ as an empirical task constant.

\subsection{From noise to bits}\label{sec:bitsbridge}

A uniform quantizer with step $\Delta$ and no clipping produces errors that are approximately uniform on
$[-\Delta/2,\Delta/2]$, with root-mean-square $\Delta/\sqrt{12}$. Its relative error is therefore
$\eta_q(b) = \Delta(b)/(\sqrt{12}\,\rms(x))$, and the dependence of $\eta_q$ on $b$ is set entirely by how
the quantizer chooses its step.

\begin{proposition}[Bits-to-noise rate]\label{prop:gamma}
Let $q_p(b)\asymp 2^{b}$ be the number of positive levels and define $\gamma_q$ by
$\eta_q(b)\propto2^{-\gamma_q b}$. Then, in the no-clipping regime,
\begin{enumerate}[label=(\alph*),itemsep=1pt]
  \item \emph{LSQ-style scaling} \citep{esser2020lsq}, $\Delta = 2\,\E|x|/\sqrt{q_p}$, gives $\gamma_q = \tfrac12$;
  \item \emph{min--max scaling}, $\Delta = \max|x|/q_p$, gives $\gamma_q = 1$.
\end{enumerate}
\end{proposition}
\begin{proof}
Substitute each step rule into
\[
  \eta_q = \frac{\Delta}{\sqrt{12}\,\rms(x)} .
\]
The factors $\E|x|/\rms(x)$ and
$\max|x|/\rms(x)$ do not depend on $b$, so $\eta_q\propto q_p^{-1/2}\propto2^{-b/2}$ in case (a) and
$\eta_q\propto q_p^{-1}\propto2^{-b}$ in case (b).
\end{proof}

We measure $\gamma = 0.53$ (MLPs), $0.55$ (CNNs) and $0.57$ (ViTs) for LSQ-style scaling and $\gamma = 1.04$
for min--max scaling (\cref{sec:bits}). The small excess over $\tfrac12$ comes from clipping at low bit-widths,
which \cref{prop:gamma} excludes.

\begin{corollary}[Bit-floor law]\label{cor:bitfloor}
If $\etac(D) = C_1 D^{-\alpha}$ and $\eta_q(b) = A_q 2^{-\gamma_q b}$, and if quantization error acts on
accuracy like Gaussian noise of equal relative size, then
\begin{equation}
  \bc(D) \;=\; \frac{\alpha}{\gamma_q}\,\log_2 D \;+\; \frac{1}{\gamma_q}\log_2\frac{A_q}{C_1}.
  \label{eq:bitfloor}
\end{equation}
\end{corollary}

The corollary separates a property of the \emph{network} ($\alpha$) from a property of the \emph{quantizer}
($\gamma_q$): the same network needs about twice as many additional bits per depth doubling under
LSQ-style scaling as under min--max scaling. Its last hypothesis is only approximate: real quantization
error is not Gaussian: at the floor, quantization error of a given size is up to about $1.5$ times as
harmful as Gaussian noise in CNNs and ViTs and about as harmful in MLPs (\cref{sec:bits}).

\subsection{Noise-aware training}

Training with the perturbation present (\emph{QAT}, or noise-aware training) changes the network, and
hence $G$ and $\Lambda$. We summarise its effect by the \emph{tolerance gain}
\begin{equation}
  g(D) \;=\; \frac{\etac^{\mathrm{QAT}}(D)}{\etac^{\mathrm{PTQ}}(D)},
  \qquad\text{so that}\qquad
  \aqat \;=\; \aptq \;-\; \frac{d\log g}{d\log D}.
  \label{eq:gain}
\end{equation}
Equation \eqref{eq:gain} is an identity. It makes the central empirical question of \cref{sec:qat}
precise: QAT \emph{flattens} the depth law if its benefit grows with depth ($g$ increasing) and
\emph{steepens} it if the benefit decays ($g$ decreasing). We find the latter.

\section{Experimental protocol}\label{sec:protocol}

\paragraph{Architectures.}
We use two model families and vary depth $D$ at fixed width (\cref{tab:archs}).
\emph{MLPs} have width 32 and $D$ weight layers ($784\to32$, then $D-2$ hidden $32\to32$ layers, then
$32\to10$). The residual MLPs are pre-activation networks with linear-ending branches,
$h_{l+1}=h_l+s\,(W_l\,\phi(h_l)+\beta_l)$, with either depth-scaled branches $s=1/\sqrt D$ or a fixed
scale $s=1/\sqrt8$; the two coincide at $D=8$, which makes the fixed-scale network a control that
isolates the scaling from the skip connection.
\emph{CNNs} have 64 channels, a stride-2 $3\times3$ stem and $D$ blocks of $3\times3$ convolutions,
followed by global average pooling and a linear head. We study plain CNNs, residual CNNs with
$s=1/\sqrt D$, $s=1/\sqrt8$ and $s=1$, and plain CNNs with Dirac (identity) initialisation, which are
trainable at large depth. The main architectures use no normalisation layers, because normalisation rescales activations layer by layer and acts as an implicit, data-dependent branch scale; we test two normalised CNN variants separately (\cref{sec:norm}).

\begin{table}[t]
\centering\small
\caption{Architectures and training. Depths are the values of $D$ used in the depth laws.}
\label{tab:archs}
\resizebox{\linewidth}{!}{\begin{tabular}{lllll}
\toprule
Family & Variant & Branch scale $s$ & Depths $D$ & Data, training \\
\midrule
MLP, width 32 & plain & -- & 4, 6, 8, 12, 16, 24 (32) & \multirow{3}{*}{\shortstack[l]{Fashion-MNIST, MNIST\\20k train, Adam, 3 epochs}} \\
 & residual & $1/\sqrt{8}$ & 8, 16, 32, 64 & \\
 & residual & $1/\sqrt{D}$ & 8, 16, 32, 64 & \\
\midrule
CNN, 64 channels & plain & -- & 4, 8, 16 & \multirow{7}{*}{\shortstack[l]{CIFAR-10, Adam + cosine,\\15 epochs (30 in control),\\flip + crop augmentation}} \\
 & residual & $1/\sqrt{D}$ & 4--64 & \\
 & residual & $1/\sqrt{8}$ & 4--64 & \\
 & residual & $1$ & 8--64 & \\
  & plain, Dirac init & -- & 8--64 & \\
 & pre-norm residual (LayerNorm) & $1$ & 4--64 & \\
 & post-norm plain (LayerNorm) & -- & 4--64 & \\
\bottomrule
\end{tabular}}
\end{table}

\paragraph{Perturbations and quantizers.}
Gaussian perturbations follow \cref{def:perturb} and are applied at every weight matrix and every
activation entering a matrix multiplication or convolution; the network input is not perturbed.
Quantization acts at exactly the same sites. Weights use a symmetric uniform quantizer per tensor and
activations an unsigned uniform quantizer per tensor. The \emph{LSQ-style} quantizer sets the step to
$2\,\E|x|/\sqrt{q_p}$ with clipping to the representable range; the \emph{min--max} quantizer sets it to
$\max|x|/q_p$. Activation steps are calibrated once on 5{,}000 training images; for MLPs we verified
that dynamic per-batch steps give indistinguishable bit floors (slopes $1.23$ vs.\ $1.24$ bits per depth
doubling for plain MLPs).

\paragraph{PTQ and QAT.}
In the \emph{PTQ} regime the network is trained at full precision and perturbed or quantized only at test
time. In the \emph{QAT} regime the same perturbation (Gaussian noise of level $\eta$, or $b$-bit
quantization with a straight-through gradient) is applied in every training forward pass and at test
time. Continuous noise is the primary instrument, because integer bit-widths resolve the floor only to
about one bit, which is too coarse to estimate exponents that differ by a fraction of a bit per depth
doubling.

\paragraph{Measurements.}
For MLPs and all noise-aware comparisons, the noise floor $\etac$ is estimated by fitting a logistic
function of $\log\eta$ to test accuracy pooled over seeds, over 5--12 noise levels spanning $0.01$--$3$,
and reading off the midpoint of \cref{def:floor}. For the PTQ floors of CNNs and ViTs it is interpolated
log-linearly per model between 14 noise levels from $0.02$ to $2$ (three draws each), and exponents are
fitted to the per-model values. The bit floor $\bc$ interpolates the midpoint between the tested bit-widths
$b\in\{1,\dots,6,8\}$. Amplification $G$ is estimated from \eqref{eq:Gdef} with $50$ noise draws per
level at $\eta\in\{0.001,0.005,0.01\}$ on $2{,}000$ test inputs; the ratio $G(0.001)/G(0.01)$ lies in
$0.98$--$1.06$ (architecture means), confirming the linear regime. Exponents are least-squares slopes in log--log coordinates;
we report 95\% $t$-intervals for single exponents and seed-bootstrap intervals for ratios and
differences of exponents. Headline MLP results use ten seeds; CNN results use three (two in the
30-epoch control).

\paragraph{Compute.}
All experiments ran on one server with two Intel Xeon Gold 6426Y processors (32 cores), 128~GB of memory
and one NVIDIA RTX A4000 GPU (16~GB), under Windows Server 2022 with Python 3.12, NumPy 2.4, PyTorch 2.5
(CUDA 12.1) and Hugging Face Transformers. The MLP experiments are CPU-only; the complete MLP grid takes
about 1.7 hours on 56 parallel processes. Measuring all trained CNNs, ViTs and the nine language models
takes about 3.5 GPU-hours; the most expensive single experiment, noise-aware training of the ViTs, takes
about 10 GPU-hours.

\paragraph{Pre-registration.}
Every quantitative prediction presented as a test in \cref{sec:results} (predicted exponents, the
residual-scaling outcomes, the 30-epoch decision rule) was written down before the corresponding runs.
Measurement artefacts found along the way are documented in \cref{app:artefacts}.

\section{Results}\label{sec:results}

\subsection{Amplification follows the growth law}\label{sec:amplification}

\Cref{fig:amp}(a) shows $G^2$ against depth for the three MLP families (width 32, ten seeds each). As
\cref{prop:growth} requires, $G^2$ grows approximately linearly in depth, with slope proportional to $s^2$
(the deepest point of each family lies somewhat above the line, which is why the fitted intercepts are not
meaningful on their own): the fitted
slope is $18.8$ [16.4, 21.3] for plain MLPs ($s=1$) and $1.10$ [0.75, 1.45] for the fixed-scale residual MLP
($s^2=1/8$, i.e.\ $8.8$ per unit $s^2$), while with $1/\sqrt D$ branch scaling $G^2$ does not grow at all
(slope $-0.09\pm0.05$). The corresponding exponents are $\rho=0.55\pm0.05$ (plain, asymptotically
$\tfrac12$), $0.29\pm0.07$ (fixed scale) and $-0.06\pm0.04$ (depth-scaled). As discussed in
\cref{rem:kappa}, the constant $\kappa$ is architecture- and width-specific; the exponents are what matter.

Trained CNNs and ViTs behave differently (\cref{fig:amp}(b,c)). In the standard variants amplification is
flat or decreases slightly with depth: training moves these networks away from the norm-preserving regime of
\cref{ass:iso} towards mild attenuation. Amplification grows in the networks built to be hard---unscaled
residual CNNs ($\rho=0.09\pm0.05$) and Dirac-initialised plain CNNs ($\rho=0.16\pm0.07$)---and in the
post-LN ViT ($\rho=0.04\pm0.01$). The depth-scaled ViT attenuates the most ($\rho=-0.12\pm0.03$).

\begin{figure}[t]
\centering
\includegraphics[width=\linewidth]{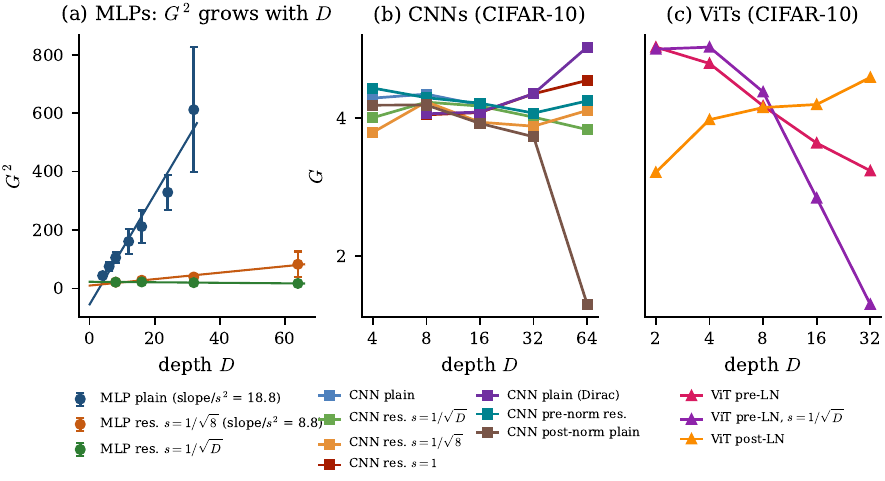}
\caption{\textbf{Amplification} (predictive metric \eqref{eq:Gkl}). (a) MLPs: $G^2$ is linear in depth
with slope $\propto s^2$ (\cref{prop:growth}); bars are standard deviations over ten seeds. (b) CNNs and
(c) ViTs on CIFAR-10 (three seeds): trained standard networks attenuate slightly with depth; unscaled
residual and Dirac-initialised CNNs amplify.}
\label{fig:amp}
\end{figure}

\subsection{Amplification predicts the depth exponent}\label{sec:predictor}

\Cref{prop:collapse} predicts $\aptq=\rho$. \Cref{fig:predictor} and \cref{tab:exponents} test this across
thirteen trained architectures in three families and the four GPT-2 models. In every case except one, the
95\% intervals of the two exponents overlap, across a range from $-0.15$ to $+0.32$, and in eleven of the
thirteen trained architectures the exponents differ by at most $0.08$: fixed-scale residual MLPs
($\rho=0.29$, $\aptq=0.32$), depth-scaled residual MLPs, six of the seven CNN variants---including the constructed
positive penalty of the Dirac-initialised CNN ($\rho=0.16$, $\aptq=0.17$)---all three ViTs, and GPT-2 from 12
to 48 layers ($\rho=-0.57$, $\aptq=-0.53$; \cref{sec:llm}). The exception is the plain MLP, which collapses
faster than its amplification predicts ($\aptq=0.70$, $\rho=0.55$); the gap closes with width
(\cref{app:width}) and is discussed in \cref{sec:discussion}. The post-norm CNN agrees only within its wide interval
($\aptq=-0.15$, $\rho=-0.36\pm0.24$), because its deepest models are undertrained.

The mechanism can be tested more directly. \Cref{prop:collapse} states that $\Lambda=\etac G$ is a
task-level constant that does not depend on depth. \Cref{fig:residual}(b) shows that it is also
remarkably stable \emph{across} architectures and families: over all thirteen trained architectures and
all depths, $\Lambda = 1.45$ with a coefficient of variation of $14\%$ (range $1.17$--$1.81$, excluding
one undertrained post-norm CNN at $D=64$), and for GPT-2 it is $1.01$--$1.13$ from 12 to 48 layers. The
noise floor of a network can therefore be predicted, to within about $15\%$, from a single full-precision
measurement: $\etac\approx1.45/G$.

\begin{figure}[t]
\centering
\includegraphics[width=0.82\linewidth]{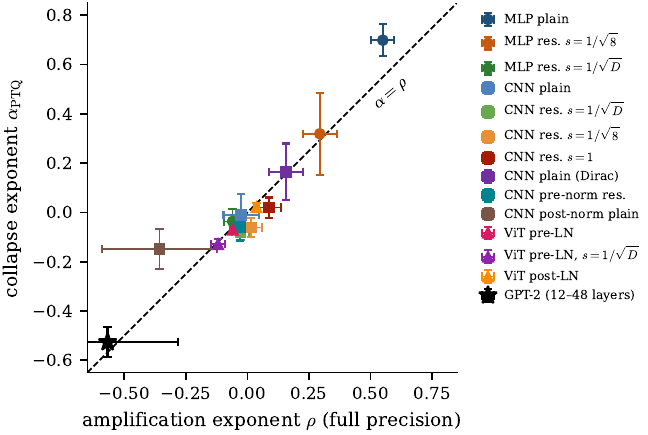}
\caption{\textbf{Amplification predicts the depth exponent.} Amplification exponent $\rho$ (full
precision, small noise) against the exponent $\aptq$ of the PTQ noise floor for thirteen trained
architectures and the GPT-2 family (star); bars are 95\% intervals. All points except the width-32 plain
MLP (top right) lie on the identity line predicted by \cref{prop:collapse}.}
\label{fig:predictor}
\end{figure}

\begin{table}[t]
\centering\small
\caption{\textbf{Depth exponents and collapse constants.} $\rho$: amplification exponent; $\aptq$: exponent
of the PTQ noise floor; $\Lambda$: range of $\etac G$ over depths; bit slope: measured bits per depth
doubling of the calibrated LSQ-style bit floor; last column: slope predicted by \cref{cor:bitfloor} with the
measured $\gamma$ ($0.53$ MLP, $0.55$ CNN, $0.57$ ViT). Brackets are 95\% intervals. The Dirac CNN has no
bit floor because its identity-shaped kernels break the LSQ-style step rule.}
\label{tab:exponents}
\resizebox{\linewidth}{!}{\begin{tabular}{llccccc}
\toprule
Family & Architecture & $\rho$ (amplification) & $\alpha_{\mathrm{PTQ}}$ (collapse) & $\Lambda=\eta_c G$ & bit slope & $\alpha_{\mathrm{PTQ}}/\gamma$ \\
\midrule
MLP & plain & $+0.55$ {\scriptsize[+0.50, +0.60]} & $+0.70$ {\scriptsize[+0.63, +0.76]} & 1.39--1.81 & $+1.24$ & $+1.31$ \\
MLP & res. $s=1/\sqrt{8}$ & $+0.29$ {\scriptsize[+0.23, +0.36]} & $+0.32$ {\scriptsize[+0.15, +0.49]} & 1.67--1.77 & $+0.42$ & $+0.60$ \\
MLP & res. $s=1/\sqrt{D}$ & $-0.06$ {\scriptsize[-0.10, -0.02]} & $-0.04$ {\scriptsize[-0.09, +0.01]} & 1.69--1.78 & $-0.02$ & $-0.07$ \\
CNN & plain & $-0.03$ {\scriptsize[-0.10, +0.05]} & $-0.01$ {\scriptsize[-0.10, +0.07]} & 1.40--1.43 & $+0.05$ & $-0.02$ \\
CNN & res. $s=1/\sqrt{D}$ & $-0.03$ {\scriptsize[-0.06, +0.01]} & $-0.08$ {\scriptsize[-0.11, -0.05]} & 1.40--1.65 & $-0.02$ & $-0.15$ \\
CNN & res. $s=1/\sqrt{8}$ & $+0.01$ {\scriptsize[-0.03, +0.06]} & $-0.06$ {\scriptsize[-0.10, -0.02]} & 1.34--1.68 & $-0.12$ & $-0.11$ \\
CNN & res. $s=1$ & $+0.09$ {\scriptsize[+0.04, +0.13]} & $+0.02$ {\scriptsize[-0.02, +0.06]} & 1.40--1.61 & $+0.30$ & $+0.04$ \\
CNN & plain (Dirac) & $+0.16$ {\scriptsize[+0.09, +0.23]} & $+0.16$ {\scriptsize[+0.05, +0.28]} & 1.30--1.46 & -- & $+0.30$ \\
CNN & pre-norm res. & $-0.03$ {\scriptsize[-0.07, +0.01]} & $-0.06$ {\scriptsize[-0.11, -0.01]} & 1.33--1.47 & $-0.08$ & $-0.11$ \\
CNN & post-norm plain & $-0.36$ {\scriptsize[-0.59, -0.12]} & $-0.15$ {\scriptsize[-0.23, -0.07]} & 0.68--1.35 & $+0.03$ & $-0.27$ \\
ViT & pre-LN & $-0.06$ {\scriptsize[-0.07, -0.05]} & $-0.07$ {\scriptsize[-0.08, -0.06]} & 1.29--1.34 & $-0.06$ & $-0.12$ \\
ViT & pre-LN, $s=1/\sqrt{D}$ & $-0.12$ {\scriptsize[-0.15, -0.09]} & $-0.13$ {\scriptsize[-0.15, -0.11]} & 1.28--1.37 & $-0.17$ & $-0.23$ \\
ViT & post-LN & $+0.04$ {\scriptsize[+0.03, +0.05]} & $+0.02$ {\scriptsize[+0.01, +0.04]} & 1.17--1.24 & $-0.12$ & $+0.04$ \\
\midrule
LLM & GPT-2, 12--48 layers & $-0.57$ {\scriptsize[-0.85, -0.28]} & $-0.53$ {\scriptsize[-0.59, -0.47]} & 1.01--1.13 & -- & -- \\
\bottomrule
\end{tabular}
}
\end{table}

\subsection{Residual scaling removes the depth penalty}\label{sec:residual}

\Cref{fig:residual}(a) shows the noise floors of the three MLP families. With depth-scaled branches the
floor does not decrease with depth: under PTQ it rises slightly from $\etac = 0.38$ at $D=8$ to $0.41$ at
$D=64$ ($\aptq = -0.04\pm0.05$, ten seeds), and under QAT it stays at $0.80$--$0.85$. The fixed-scale control,
identical at $D=8$ by construction, loses tolerance steadily ($\aptq = 0.32$, $\etac$ from $0.38$ to $0.20$).
The difference is caused by the depth dependence of the branch scale, not by the skip connection, as
\cref{prop:growth}(b)--(c) predicts. In CNNs the depth-scaled residual network is likewise flat; the
fixed-scale CNN is flat too, but so is its amplification, so the law predicts no penalty and none is observed.

\begin{figure}[t]
\centering
\includegraphics[width=\linewidth]{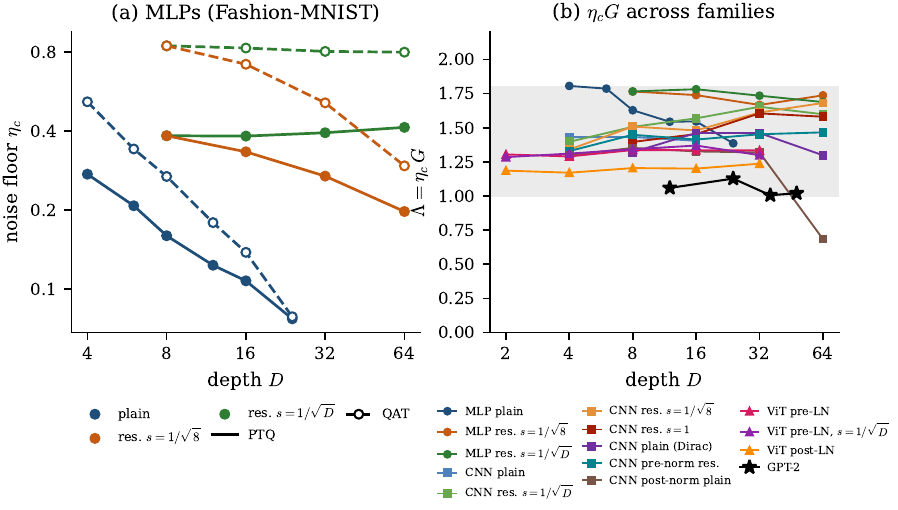}
\caption{\textbf{Residual scaling and the collapse constant.} (a) MLP noise floors under PTQ (solid) and QAT
(dashed): depth-scaled residual branches ($s=1/\sqrt D$) make the floor depth-independent; the fixed-scale
control does not. (b) $\Lambda=\etac G$ across depth for thirteen trained architectures in three families and
for GPT-2 (stars); the grey band marks $1.0$--$1.8$.}
\label{fig:residual}
\end{figure}

\subsection{Normalisation}\label{sec:norm}

Deployed networks are normalised, and normalisation rescales activations at every layer, so it could in
principle override the architecture effects above. We therefore trained two CIFAR-10 CNN variants with
LayerNorm over channels, keeping the normalisation layers at full precision as is common in deployment:
a \emph{pre-norm residual} network, $h_{l+1} = h_l + \mathrm{conv}(\phi(\mathrm{Norm}(h_l)))$ with unit
branch scale (the block used by pre-norm ResNets and transformers), and a \emph{post-norm plain} network,
$h_{l+1}=\phi(\mathrm{Norm}(\mathrm{conv}(h_l)))$. We fixed three predictions in advance. (1) In the
pre-norm network, each branch receives a normalised input and returns an $O(1)$ output, while the stream
norm grows like $\sqrt l$; branch $l$ then contributes $O(1/l)$ to $G^2$, so $G^2$ grows only
logarithmically and the floor should be nearly depth-independent---normalisation acts as an implicit
$1/\sqrt l$ branch scale. (2) In the post-norm plain network, normalisation preserves relative
perturbations, so $G$ should grow as in plain networks before training effects. (3) In both,
$\aptq\approx\rho$.

Predictions (1) and (3) hold; prediction (2) does not. The pre-norm residual CNN has
$\rho = -0.03\pm0.04$ and $\aptq = -0.06\pm0.05$, a flat bit floor ($2.5$--$2.7$ bits from $D=8$ to $64$,
slope $-0.08$ against a predicted $-0.11$) and stable accuracy ($0.67$--$0.69$) at all depths. The
post-norm plain CNN, like every other trained standard CNN, attenuates slightly after training
($\aptq = -0.15\pm0.08$; $\rho = -0.36\pm0.24$, a value dominated by its undertrained $D=64$ models), so the predictor holds but the untrained-regime
prediction of growth does not survive training. For both networks $\Lambda=\etac G$ lies in
$1.30$--$1.47$ (excluding the undertrained $D=64$ post-norm network), the same range as all other architectures (\cref{fig:predictor,fig:residual}). One
caveat: the post-norm plain CNN trains poorly at large depth (accuracy $0.54$ at $D=32$ and $0.27$ at
$D=64$), and its bit floor at those depths inherits the excess quantization harm we observed for other
undertrained deep models (\cref{sec:bits}).

\subsection{Transformers}\label{sec:vit}

We trained Vision Transformers on CIFAR-10 (patch size 4, width 128, four heads, MLP ratio 2, 30 epochs,
three seeds) at depths 2--32 in three variants: standard pre-LN blocks, pre-LN blocks with depth-scaled
branches ($s=1/\sqrt D$, as in depth-$\mu$P), and the original post-LN blocks. Quantization and noise act on
the weights and inputs of every linear layer (QKV, output projection, both MLP layers, patch embedding and
head); softmax and LayerNorm stay at full precision. Transformer activations are signed, so activations use
a signed LSQ-style quantizer, whose rate is again $\gamma\approx\tfrac12$ (measured $0.57$).

All three ViTs obey the law closely (\cref{tab:exponents}): $\rho=-0.06$ vs.\ $\aptq=-0.07$ (pre-LN),
$-0.12$ vs.\ $-0.13$ (pre-LN with depth scaling) and $+0.04$ vs.\ $+0.02$ (post-LN), with
$\Lambda=1.17$--$1.37$ at all depths. Accuracy is stable across depth ($0.76$--$0.84$), so these are also our
cleanest measurements, free of the undertraining confound. The depth-scaled ViT gives the most striking
single result: its amplification \emph{falls} with depth, so the tolerable noise \emph{rises} (from
$0.40$ at $D=2$ to $0.55$ at $D=32$) and the bit floor drops by $0.7$ bits (from $3.0$ to $2.3$). In a well-scaled transformer,
\emph{deeper can mean fewer bits}.

\subsection{Pretrained language models}\label{sec:llm}

All networks so far were trained by us. To test the law on models we did not design, we measured nine
publicly released language models without any training: GPT-2 small, medium, large and XL (12, 24, 36 and 48
layers) and Pythia-70M, 160M, 410M, 1B and 1.4B (6--24 layers). Noise and quantization act on the weights and
inputs of every linear layer inside the transformer blocks; embeddings and the output head are excluded,
because GPT-2 ties them. Accuracy is next-token top-1 accuracy on the WikiText-103 validation set, and the
floor is the noise level at which it halves. \Cref{tab:llm} and \cref{fig:llm} summarise the results.

\paragraph{The metric matters.} Measured with the logit norm \eqref{eq:Gdef}, amplification varied erratically
across the nine models and $\etac G$ ranged over two orders of magnitude. The reason is visible in the
data: in the two smallest Pythia models the logit-norm $G$ is three to six times smaller than the $G$ of
mean-centred logits, i.e.\ most of the logit vector moves along the softmax-invariant common direction, and in
GPT-2 medium it is almost three times larger. With the predictive definition
\eqref{eq:Gkl} the picture changes completely.

\paragraph{GPT-2 obeys the law across depth.} From 12 to 48 layers, the noise floor of GPT-2 doubles
($\etac$ from $0.146$ to $0.301$) while its amplification halves ($G$ from $7.3$ to $3.4$), so that
$\Lambda = 1.06, 1.13, 1.01, 1.02$ is constant to within $6\%$ and the exponents agree
($\aptq=-0.53\pm0.06$, $\rho=-0.57\pm0.29$; \cref{fig:llm}(a)). Deeper GPT-2 models are more robust
to perturbation, and full-precision amplification predicts by how much.

\paragraph{A validity condition.} \Cref{prop:collapse} is a first-order statement, and its premise can be
checked on any model: $G$ must not depend on the noise level used to measure it. All trained networks in
this paper pass ($G(0.001)/G(0.01)=0.98$--$1.06$, architecture means), as do all GPT-2 models and
Pythia-410M, 1B and 1.4B ($G(0.002)/G(0.01)=0.99$--$1.01$). The two
smallest Pythia models fail ($1.56$ and $1.87$): they respond nonlinearly even to $0.2\%$ noise, and the law
makes no claim about them. Among the models that pass, Pythia-1B and 1.4B have $\Lambda=1.21$ and $1.52$,
inside the range of our trained networks; Pythia-410M is an outlier ($\Lambda=3.7$ per tensor, $1.8$ with
per-channel perturbations), which we cannot explain. Per-channel perturbation, which neutralises activation
outliers, leaves the GPT-2 results essentially unchanged ($\Lambda=1.01$--$1.28$).

\begin{figure}[t]
\centering
\includegraphics[width=\linewidth]{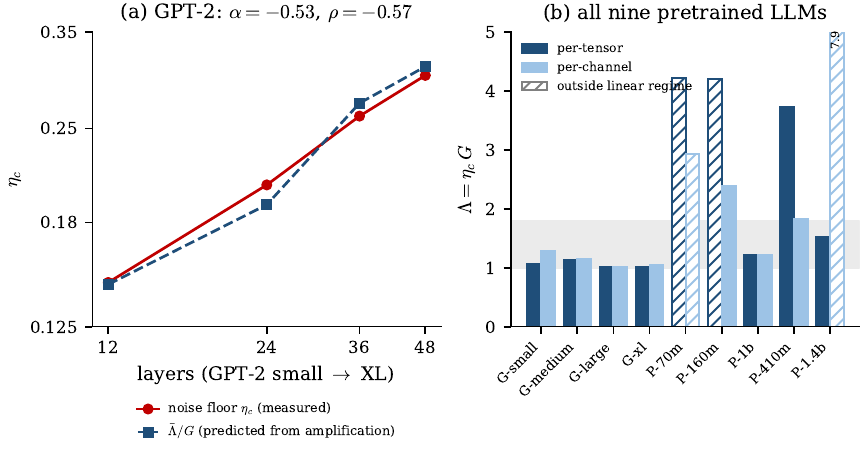}
\caption{\textbf{Pretrained language models (no training).} (a) GPT-2 from 12 to 48 layers: the measured
noise floor (red) and the prediction $\bar\Lambda/G$ from full-precision amplification (blue) coincide.
(b) $\Lambda=\etac G$ for nine pretrained models, per-tensor and per-channel perturbations; hatched bars fail
the linearity condition. The grey band is the range of our trained networks.}
\label{fig:llm}
\end{figure}

\begin{table}[t]
\centering\small
\caption{\textbf{Pretrained language models} (per-tensor perturbations). Accuracy: next-token top-1 on
WikiText-103; linearity: $G(0.002)/G(0.01)$; values in parentheses fail the linearity condition.}
\label{tab:llm}
\begin{tabular}{lrrccccc}
\toprule
Model & layers & width & accuracy & $\eta_c$ & $G$ & linearity & $\Lambda=\eta_c G$ \\
\midrule
GPT-2 small & 12 & 768 & 0.387 & 0.146 & 7.26 & 1.00 & 1.06 \\
GPT-2 medium & 24 & 1024 & 0.429 & 0.205 & 5.49 & 1.01 & 1.13 \\
GPT-2 large & 36 & 1280 & 0.442 & 0.261 & 3.85 & 1.00 & 1.01 \\
GPT-2 XL & 48 & 1600 & 0.449 & 0.301 & 3.39 & 1.00 & 1.02 \\
Pythia-70m & 6 & 512 & 0.339 & 0.050 & 84.82 & 1.87 & \textit{(4.22)} \\
Pythia-160m & 12 & 768 & 0.396 & 0.080 & 52.70 & 1.56 & \textit{(4.21)} \\
Pythia-1b & 16 & 2048 & 0.473 & 0.218 & 5.55 & 1.00 & 1.21 \\
Pythia-410m & 24 & 1024 & 0.457 & 0.126 & 29.69 & 1.01 & 3.73 \\
Pythia-1.4b & 24 & 2048 & 0.486 & 0.202 & 7.52 & 1.01 & 1.52 \\
\bottomrule
\end{tabular}

\end{table}

\subsection{From noise to bits}\label{sec:bits}

\Cref{fig:bits}(a) shows the relative quantization error $\eta_q(b)$ of the two step rules. The measured
bits-to-noise rates are $\gamma = 0.535$ (LSQ-style, MLPs), $0.550$ (LSQ-style, CNNs), $0.572$ (signed
LSQ-style, ViTs) and $1.039$ (min--max, CNNs), against $\tfrac12$ and $1$ from \cref{prop:gamma}. The small excess of the
LSQ-style rate comes from clipping at low bit-widths. The same network therefore needs roughly twice as
many extra bits per depth doubling under LSQ-style scaling as under min--max scaling.

\Cref{fig:bits}(b) and the last two columns of \cref{tab:exponents} test \cref{cor:bitfloor}. For MLPs the
calibrated bit slopes are $1.24$ (plain), $0.42$ (fixed scale) and $-0.02$ (depth-scaled), against
predicted $\aptq/\gamma$ of $1.31$, $0.60$ and $-0.07$. The fixed-scale prediction carries the wide
interval of its $\aptq$ ([0.15, 0.49], giving [0.28, 0.91] bits per doubling) and is consistent with the
measurement. For CNNs and ViTs all slopes are small and agree in sign and size with the prediction, with one
exception: the unscaled residual CNN needs $+0.30$ bits per doubling against a predicted $+0.04$.
Two further measurements locate the cause in the quantizer rather than the network. With min--max scaling
the same models give $+0.10$ against a predicted $+0.02$; and after 30 instead of 15 epochs of training,
the LSQ-style slope falls to $+0.08$ while the noise exponent stays near zero ($+0.03$). The excess penalty
therefore comes from mean-based clipping acting on the poorly converged activations of the deepest,
undertrained models. At the bit floor, the relative size of the quantization error compared with the
Gaussian floor, $\eta_q(\bc)/\etac$, is $0.64$--$1.02$ in CNNs and ViTs (lowest in the deepest,
undertrained models; $0.51$ for the post-norm CNN at $D=64$) and $0.93$--$1.12$ in MLPs: quantization is
somewhat more harmful than Gaussian noise of the same size in the CIFAR-10 networks, and about as harmful in
the MLPs. This shifts the intercept of \eqref{eq:bitfloor} but, except in the clipping case, not its slope.

\begin{figure}[t]
\centering
\includegraphics[width=\linewidth]{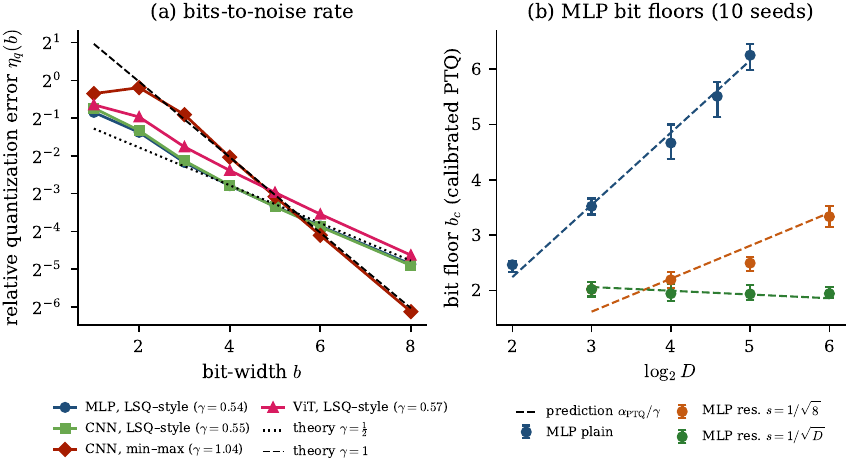}
\caption{\textbf{From noise to bits.} (a) Relative quantization error against bit-width. LSQ-style scaling
gives $\gamma\approx\tfrac12$ and min--max scaling $\gamma\approx1$, as predicted by \cref{prop:gamma}.
(b) Calibrated PTQ bit floors of MLPs (ten seeds; bars are seed-bootstrap 95\% intervals) with the slope
predicted by \cref{cor:bitfloor} from the measured noise exponent (dashed).}
\label{fig:bits}
\end{figure}

\subsection{The QAT paradox}\label{sec:qat}

Noise-aware training makes networks much more robust: at small depth it raises the tolerable noise
by a factor $g\approx1.8$--$2.2$ in every architecture we trained with noise---plain and both residual MLPs
on two datasets, the unscaled residual CNN and the pre-LN ViT (\cref{fig:qat}). If this benefit were
depth-independent, $g$ would be constant and, by the identity \eqref{eq:gain}, the QAT and PTQ depth
exponents would coincide. They do not. In every architecture we trained with noise, including the two
without a PTQ depth penalty, the gain decays with depth, so QAT raises the floor's level but steepens its
slope.

\begin{figure}[t]
\centering
\includegraphics[width=\linewidth]{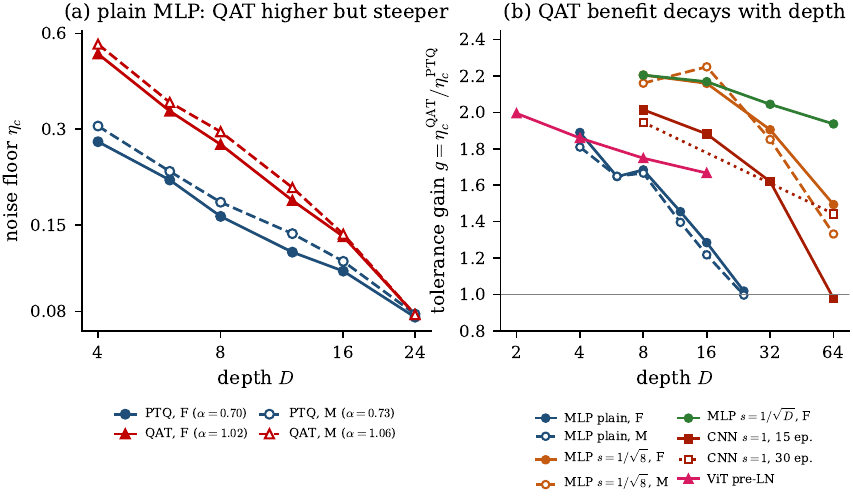}
\caption{\textbf{The QAT paradox.} (a) Plain MLP noise floors under PTQ (blue) and QAT (red) on
Fashion-MNIST (F, filled) and MNIST (M, open), ten seeds each. QAT is more tolerant at every depth but its floor
falls faster: $\aqat\approx1.5\,\aptq$ on both datasets. (b) The tolerance gain $g(D)$ of QAT over PTQ
decays with depth in every architecture, including the depth-scaled residual MLP and the pre-LN ViT, which
have no PTQ depth penalty. For the unscaled residual CNN, doubling the training budget (dotted) recovers
part, but not all, of the gain at $D=64$.}
\label{fig:qat}
\end{figure}

\begin{table}[t]
\centering\small
\caption{\textbf{QAT steepens the depth law.} Exponents of the noise floor with 95\% $t$-intervals; ratio
and difference with seed-bootstrap 95\% intervals. The ratio is undefined (n/a) when $\aptq\approx0$. $^\dagger$CNN rows use three seeds (15 epochs, four depths) and two seeds (30 epochs, depths 8 and 64 only); their exponents have no intervals and are indicative. For the
CNN, exponents are fitted over $D\in\{8,16,32,64\}$ (15 epochs) and $D\in\{8,64\}$ (30 epochs); for the ViT
over $D\in\{2,4,8,16\}$ (three seeds).}
\label{tab:qat}
\resizebox{\linewidth}{!}{\begin{tabular}{llccccc}
\toprule
Dataset & Architecture & seeds & $\alpha_{\mathrm{PTQ}}$ & $\alpha_{\mathrm{QAT}}$ & $\alpha_{\mathrm{QAT}}/\alpha_{\mathrm{PTQ}}$ & $\Delta\alpha$ \\
\midrule
Fashion-MNIST & MLP plain & 10 & $0.70\pm0.07$ & $1.02\pm0.11$ & $1.47$ {\scriptsize[1.37, 1.57]} & $+0.33$ {\scriptsize[+0.27, +0.37]} \\
MNIST & MLP plain & 10 & $0.73\pm0.08$ & $1.06\pm0.15$ & $1.45$ {\scriptsize[1.37, 1.54]} & $+0.33$ {\scriptsize[+0.29, +0.37]} \\
Fashion-MNIST & MLP res. $s=1/\sqrt{8}$ & 10 & $0.32\pm0.17$ & $0.51\pm0.39$ & $1.58$ {\scriptsize[1.39, 1.83]} & $+0.19$ {\scriptsize[+0.14, +0.23]} \\
MNIST & MLP res. $s=1/\sqrt{8}$ & 10 & $0.28\pm0.12$ & $0.51\pm0.40$ & $1.86$ {\scriptsize[1.62, 2.17]} & $+0.24$ {\scriptsize[+0.20, +0.28]} \\
Fashion-MNIST & MLP res. $s=1/\sqrt{D}$ & 10 & $-0.04\pm0.05$ & $0.03\pm0.02$ & n/a & $+0.06$ {\scriptsize[+0.04, +0.09]} \\
CIFAR-10 & CNN res. $s=1$ (15 epochs)$^\dagger$ & 3 & $0.04$ & $0.38$ & n/a & $+0.33$ \\
CIFAR-10 & CNN res. $s=1$ (30 epochs)$^\dagger$ & 2 & $0.01$ & $0.15$ & n/a & $+0.14$ \\
CIFAR-10 & ViT pre-LN & 3 & $-0.09\pm0.02$ & $0.00\pm0.04$ & n/a & $+0.09$ {\scriptsize[+0.07, +0.10]} \\
\bottomrule
\end{tabular}
}
\end{table}

\Cref{tab:qat} quantifies the effect. For plain MLPs the ratio of exponents is $1.47$ [1.37, 1.57] on
Fashion-MNIST and $1.45$ [1.37, 1.54] on MNIST: two independent datasets, ten seeds each, the same number.
For fixed-scale residual MLPs it is $1.58$ [1.39, 1.83] and $1.86$ [1.62, 2.17]. Expressed as a
difference, noise-aware training adds $\Delta\alpha = +0.19$ to $+0.33$ to the depth exponent. Even the
depth-scaled residual MLP, whose PTQ floor is flat, shows a small but significant steepening
($\Delta\alpha = +0.06$ [0.04, 0.09], ten seeds): its gain decays from $2.2\times$ at $D=8$ to $1.9\times$ at $D=64$.
In plain MLPs the gain has vanished by $D=24$ ($g = 1.02$ on Fashion-MNIST, $1.00$ on MNIST).

The ratio is not a universal constant: it differs between architectures and is undefined when
$\aptq\approx0$. The robust statement is the sign, $\Delta\alpha>0$, together with a gain that decays with
depth.

\paragraph{Transformers.}
We repeated the comparison for pre-LN ViTs at depths 2--16 (three seeds, one network trained for 30 epochs
per noise level). This architecture has no PTQ depth penalty: its floor rises slightly with depth
($\aptq=-0.09\pm0.02$). The QAT pattern is nevertheless the same as in the depth-scaled residual MLP.
Noise-aware training doubles the tolerable noise at $D=2$ ($g=2.0$), the gain decays steadily to $1.67$ at
$D=16$, the QAT floor stays flat ($\aqat=0.00\pm0.04$), and $\Delta\alpha=+0.09$ [0.07, 0.10]. Noise-aware
training therefore steepens the depth law in transformers as well; here it cancels the gain in tolerance
that depth brings under PTQ.

\paragraph{Training budget.}
Could the decay simply mean that deeper noise-aware networks need more training? For the unscaled
residual CNN at 15 epochs, the gain falls from $2.0\times$ at $D=8$ to $1.0\times$ at $D=64$
($\Delta\alpha = +0.33$), and the $D=64$ models are visibly undertrained (full-precision accuracy $0.63$
against $0.77$ at $D=8$). Our pre-registered control doubled the budget to 30 epochs. Full-precision
accuracy at $D=64$ rose to $0.73$ and the gain at $D=64$ recovered to $1.44\times$, while the gain at $D=8$
was unchanged ($1.95\times$); $\Delta\alpha$ fell from $+0.33$ to $+0.14$. The outcome lies between the two
pre-registered alternatives (full recovery to $\approx2\times$, or none): in this CNN, about half of the
apparent steepening is a training-budget effect and about half persists with twice the training. In MLPs
the same caveat applies with less force: full-precision accuracy of plain MLPs is stable up to $D=16$ and
drops only at $D=24$ ($\approx0.81$ to $0.74$), yet the gain has already fallen from $1.9\times$ to
$1.3\times$ by $D=16$, and the effect is quantitatively identical on two datasets. A budget component at
the largest depths cannot be excluded, and we state the QAT result accordingly.

\subsection{Where the QAT paradox comes from}\label{sec:mechanism}
(The per-layer analysis in this subsection uses the logit-norm sensitivities of \cref{prop:sum}; the
factorisation below holds for either metric.)

Because $\etac = \Lambda/G$ (\cref{prop:collapse}), the tolerance gain of QAT factorises exactly into an
amplification factor and a collapse factor,
\begin{equation}
  g(D) \;=\; \underbrace{\frac{G_{\mathrm{PTQ}}(D)}{G_{\mathrm{QAT}}(D)}}_{\text{first-order robustness}}
  \;\times\;
  \underbrace{\frac{\Lambda_{\mathrm{QAT}}(D)}{\Lambda_{\mathrm{PTQ}}(D)}}_{\text{collapse regime}} .
  \label{eq:factor}
\end{equation}
To measure the first factor layer by layer, we injected small noise ($\eta=0.01$, 50 draws) into the
sites of one layer at a time and computed each layer's sensitivity $q_l$, whose sum is $G^2$
(\cref{prop:sum}). We did this for plain MLPs at $D\in\{4,8,16,24\}$ (five seeds), comparing networks
trained without noise to networks trained with noise at the PTQ collapse level of the same depth.
Two candidate mechanisms from our pre-registered list make opposite predictions here: finite repair
capacity predicts that QAT lowers $q_l$ in only a few layers, and credit assignment predicts that it
lowers $q_l$ mainly in late layers.

\Cref{fig:mechanism}(a) rejects both. Under PTQ every hidden layer contributes $q_l\approx2$--$3$, which is
the growth law of \cref{prop:growth} resolved layer by layer. QAT lowers the sensitivity of \emph{every}
layer, by a factor of two to five, and it lowers it \emph{least} in the last third of the network
(ratio $0.38$--$0.53$ against $0.25$--$0.32$ in the hidden layers of the first two thirds). Summed over layers, QAT reduces
$G^2$ by a factor that does not depend on depth ($0.30$, $0.34$, $0.40$ and $0.33$ at $D=4$, $8$, $16$ and
$24$), so the amplification factor in \eqref{eq:factor} is essentially constant, $1.6$--$1.8$
(\cref{fig:mechanism}(b)). A constant amplification factor alone would give a constant gain and no
paradox. The depth dependence therefore lies entirely in the collapse factor
$\Lambda_{\mathrm{QAT}}/\Lambda_{\mathrm{PTQ}}$, which falls from $1.03$ at $D=4$ to $0.59$ at $D=24$:
\emph{QAT buys first-order robustness uniformly across depth, but deep QAT networks break down earlier in
the large-perturbation regime where \eqref{eq:firstorder} no longer holds.} This places the paradox in the
same regime as the plain-MLP exception to the predictor (\cref{sec:predictor}), where $\Lambda$ also drifts
downward with depth. Because the networks for the layer map were trained at a single noise level per depth,
whereas the gains in \cref{sec:qat} come from training a separate network at each noise level, the
factorisation here is approximate; the qualitative conclusion does not depend on it.

\begin{figure}[t]
\centering
\includegraphics[width=\linewidth]{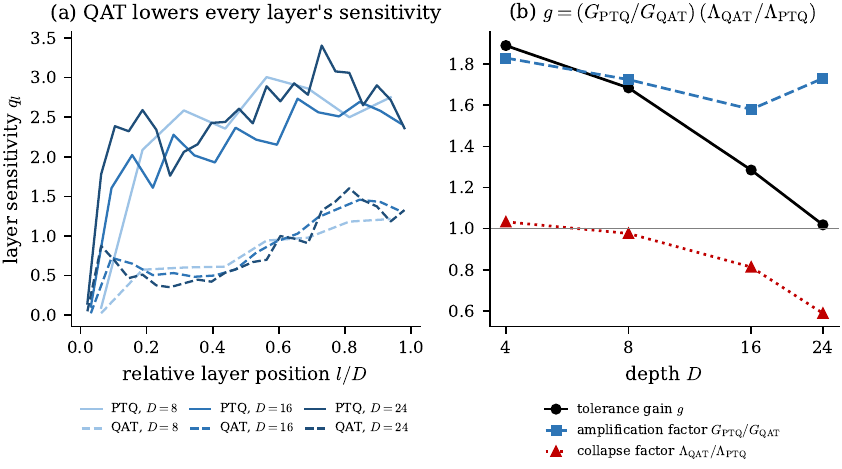}
\caption{\textbf{Mechanism of the QAT paradox (plain MLPs, five seeds).} (a) Per-layer sensitivity $q_l$
(whose sum is $G^2$) against relative position, for PTQ (solid) and QAT (dashed) networks. QAT lowers every
layer's sensitivity, least in the last layers. (b) Factorisation \eqref{eq:factor} of the tolerance gain:
the amplification factor is depth-independent, so the decay of the gain comes entirely from the collapse
factor $\Lambda_{\mathrm{QAT}}/\Lambda_{\mathrm{PTQ}}$.}
\label{fig:mechanism}
\end{figure}

\section{What does not set the floor}\label{sec:negative}

We report three natural hypotheses that the data reject, because each would otherwise be a plausible
alternative explanation of the laws above.

\paragraph{Decision margins.}
A common intuition is that the precision a classifier needs is set by how much data lies near its decision
boundary. We manipulated margins directly at fixed depth ($D=8$ and $16$) with label smoothing
($0$--$0.5$) and label noise ($0$--$40\%$), which changed the median normalised margin
(top-two logit gap over the root-mean-square logit) by up to a factor of $5$ and the fraction of low-margin
test points by up to a factor of $6$. \Cref{fig:negative}(a) shows the effect on the noise floor. Under
noise-aware training the floor did not move (all conditions within $\pm12\%$ of baseline, with
$\etac\cdot D = 2.20\pm0.11$ across all twelve conditions); under PTQ, larger margins bought $9$--$16\%$
more tolerance at $D=8$ (at most $0.4$ bits with the LSQ-style quantizer) and at most $12\%$ at $D=16$,
where label smoothing gave no measurable gain at all. Depth changed the floor by far more. In \cref{prop:collapse}, this means that the
task constant $\Lambda$ is not governed by the median margin, and that the margin-corrected predictor
$\etac\propto M/G$ fails: in eleven of the thirteen trained architectures it predicts the depth exponent
worse than $1/G$ alone, and it helps only in the two architectures where $1/G$ is least accurate (the
plain MLP and the undertrained post-norm CNN).

\paragraph{Cross-layer error cancellation.}
\citet{chen2026ptq} attribute the success of PTQ in language models to \emph{counteraction}: new
quantization error partially cancels error inherited from earlier layers. We decomposed the error at
every layer into an inherited and a new component and measured the cancellation fraction
$c=-2\langle P,N\rangle/\lVert N\rVert^2$ at 6 bits, in the three MLP families (PTQ and QAT, three seeds,
depths 8--64) and in all saved CNNs (PTQ with calibrated steps, three seeds, depths 4--64; the Dirac CNN is
excluded because 6-bit LSQ-style quantization already destroys it). In MLPs $c = -0.01\pm0.11$, with no
systematic difference between PTQ and QAT and no trend with depth. In CNNs the median over models is
$-0.05$ (interquartile range $-0.13$ to $0.02$), and where $c$ departs from zero it is mostly negative:
new and inherited errors add rather than cancel. Counteraction therefore does not explain the laws in our
networks. Instead, QAT plain MLPs \emph{damp} inherited error: the per-layer gain of the inherited
component is $0.96$--$0.98$ under QAT against $1.01$--$1.05$ under PTQ. In residual MLPs QAT does not
change this gain.

\paragraph{Activation heavy tails.}
\citet{residualfree2026} argue that residual connections make activations heavy-tailed and therefore
harder to quantize. In our residual CNNs the excess kurtosis of the activations entering the quantizers
indeed grows strongly with depth (\cref{fig:negative}(b)), reaching $100$--$120$ at $D=64$ for both the
fixed-scale ($s=1/\sqrt8$) and unscaled ($s=1$) networks. Yet only the latter showed an excess
quantization penalty in its bit slope, and that excess disappeared with longer training (\cref{sec:bits});
in the fixed-scale network the ratio $\eta_q(\bc)/\etac$ stayed at $0.7$--$0.85$ at all depths. Heavy tails
are therefore not by themselves sufficient to make quantization disproportionately harmful. Among the
unnormalised residual networks, depth-scaled branches keep the tails lightest (kurtosis $29$ at $D=64$);
pre-normalisation keeps them near $8$ at every depth.

\begin{figure}[t]
\centering
\includegraphics[width=\linewidth]{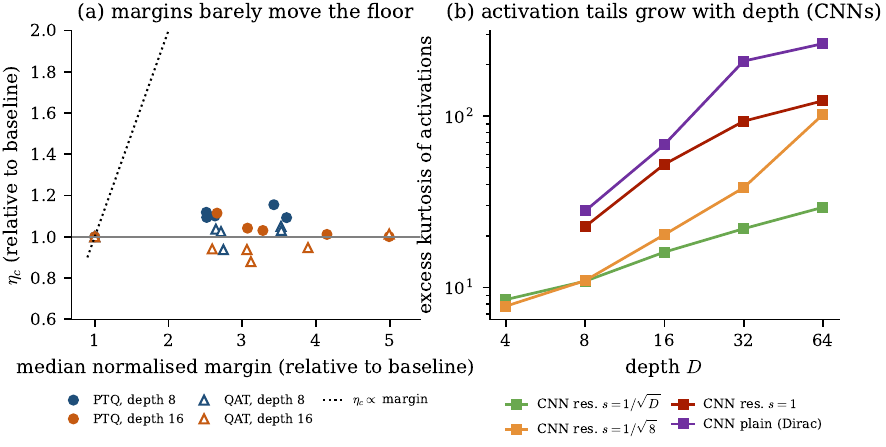}
\caption{\textbf{What does not set the floor.} (a) Label smoothing and label noise change the median
normalised margin by up to $5\times$ but move the noise floor by at most $\pm16\%$; the dotted line
shows the proportional response a margin-controlled floor would have. (b) Activation kurtosis grows with
depth in all residual and Dirac-initialised CNNs, but extra quantization harm appears only transiently in
one of them.}
\label{fig:negative}
\end{figure}

\section{Discussion and limitations}\label{sec:discussion}

\paragraph{What the laws say, and what they do not.}
The precision floor of a network is not a property of its size or its data alone: at fixed width, the
architecture decides whether depth costs precision at all, and where it does, the cost is predicted by a
full-precision measurement. Two practical consequences follow. First, the noise floor of a model can be
estimated in seconds, without any quantized simulation, as $\etac\approx\Lambda/G$ with
$\Lambda\approx1.45$ for the networks we trained and $\approx1$ for GPT-2; \cref{cor:bitfloor} then gives the
bit budget for any quantizer whose rate $\gamma_q$ is known. Second, depth-scaled residual branches and
pre-normalisation are not only stabilisers of training \citep{hayou2023width,bordelon2024depthwise} but also
make precision requirements depth-independent. The laws are empirical power laws in the sense of
\citet{kaplan2020scaling}: their exponents are architecture-specific rather than universal constants, and the
theory of \cref{sec:framework} explains their origin under stated assumptions rather than deriving them from
first principles. Note that $\Lambda$ need not be identical across tasks: it differs between our ten-class
image models and next-token prediction by a factor of about $1.4$, but within a task and family it is
constant across depth.

\paragraph{The measurement rule and the validity condition.}
Two lessons from the language models apply generally. The logit norm is a misleading measure of
amplification whenever most of the logit vector is irrelevant to the prediction; the predictive definition
\eqref{eq:Gkl} removes this problem and, for ten-class models, tracks the collapse exponent slightly better
than the logit norm does. And because the law is first-order, it should only be applied to models whose
amplification does not depend on the probe noise level; the ratio $G(0.002)/G(0.01)$ is a cheap diagnostic.
The two smallest Pythia models fail it, presumably because of their extreme activation outliers, which make
even $0.2\%$ perturbations act nonlinearly. Pythia-410M passes the check but has $\Lambda\approx2$--$4$; we do
not have an explanation.

\paragraph{The plain-MLP exception.}
Plain MLPs of width 32 collapse faster ($\aptq=0.70$) than their amplification predicts ($\rho=0.55$), and
correspondingly $\Lambda$ drifts downward with depth. At width 128 the gap shrinks to $0.53$ versus $0.47$
(\cref{app:width}), so the exception is largely a narrow-network effect. Narrow plain networks have the
largest amplification of all our architectures, so at their floor the output perturbation is largest relative
to the signal, which is where the second-order terms neglected in \eqref{eq:firstorder} matter most.

\paragraph{Towards a theory of the QAT paradox.}
\Cref{sec:mechanism} narrows the question considerably. Of the mechanisms we considered, finite repair
capacity and credit assignment are ruled out by the layer map: QAT lowers the sensitivity of every layer, by a
depth-independent total factor, and least in the last layers. What remains is that the collapse constant
$\Lambda$ of QAT networks decreases with depth relative to that of PTQ networks. $\Lambda$ is set at the
collapse point, where perturbations are large and the first-order expansion fails, so a theory of the paradox
must be a theory of the \emph{nonlinear} breakdown of trained networks under noise; the plain-MLP exception
points the same way. Two routes seem promising: a second-order extension of \cref{prop:sum}, which would give
the leading correction to $\etac=\Lambda/G$; and the observation that relative activation noise is Gaussian
dropout, whose implicit regulariser in deep linear networks is known in closed form \citep{mianjy2018implicit}.
The problem was originally motivated by the error threshold of evolutionary biology
\citep{eigen1971selforganization}, in which selection maintains information only while the per-site error rate
stays below a threshold inversely proportional to genome length; plain MLPs under QAT indeed show
$\aqat\approx1$, but the residual results rule out a universal $1/D$ law, so we keep the analogy as
motivation only.

\paragraph{Relation to residual-free transformers.}
\citet{residualfree2026} report that residual connections hurt low-bit quantization in transformers. Our
results suggest a refinement: what matters for the depth dependence is how the residual stream is scaled.
Unnormalised, unscaled branches let amplification and activation tails grow with depth; $1/\sqrt D$ branch
scaling or pre-normalisation keeps amplification flat, and our pre-LN ViTs need no more bits at depth 32 than
at depth 2. The quantizer-level harm of heavy tails, which \citet{residualfree2026} identify, is a separate
effect that our per-tensor quantizers expose mainly in undertrained models (\cref{sec:bits}).

\paragraph{Limitations.}
Our trained networks are small (width-32 MLPs, 64-channel CNNs, width-128 ViTs on CIFAR-10) and trained
briefly; the pretrained language models reach 1.5 billion parameters but are only evaluated, so the QAT
paradox is established for MLPs, one CNN and one small ViT, but not for large pretrained transformers. We test BatchNorm nowhere and
normalisation only as LayerNorm. CNN and ViT results use three seeds, and several exponents rest on four or
five depths; the fixed-scale residual MLP is not a clean power law (its local exponent rises with depth, as
\cref{prop:growth}(c) predicts). The growth constant $\kappa$ depends on architecture and width and is not
predicted by our theory. We study one quantizer family with two step rules and uniform bit-widths across
layers; at the floor, real quantization error is up to about twice as harmful as Gaussian noise of equal RMS
size in the CIFAR-10 networks and about as harmful in the MLPs, which the bridge absorbs into an intercept. Finally, the language-model comparison mixes depth with width (larger
GPT-2 models are also wider), so the GPT-2 result shows that amplification predicts the floor across the
family, not that depth alone causes the trend.

\section{Conclusion}\label{sec:conclusion}

How many bits a network needs depends on how deep it is, and in a predictable way. Across MLPs, CNNs, Vision
Transformers and pretrained GPT-2 models, the precision floor follows a power law in depth whose exponent is
the depth exponent of the network's predictive amplification. The collapse constant $\Lambda=\etac G$ is
nearly the same across architectures and depths, so the floor can be read off a full-precision measurement,
and a simple linearity check tells when this reading is valid. Architecture decides whether depth costs
precision at all: depth-scaled residual branches and pre-normalisation remove the penalty, and a well-scaled
transformer can need fewer bits as it grows deeper. Each quantizer turns noise into bits at a rate fixed by
its step rule. Noise-aware training raises the tolerable noise but steepens its depth dependence---a paradox
that replicates on two datasets and in a Vision Transformer, survives in part a doubled training budget, and
lives in the nonlinear collapse regime rather than in amplification. Explaining that regime, and testing the
laws on quantization-aware training of large transformers, are the natural next steps.

\paragraph{Reproducibility.}
All result files, the training and measurement code, and the script that generates every figure, table and
number in this paper from those files are available at \url{https://github.com/Ahmadtr/depth-laws-precision}.

\appendix
\section{Derivation of the growth law (\cref{prop:growth})}\label{app:growth}

We give the derivation for pre-activation residual networks; plain networks are the special case
discussed at the end.

\paragraph{Setting.} The stream evolves as $h_{l+1} = h_l + s\,u_l$ with branch output
$u_l = W_l\,\phi(h_l) + \beta_l$, $\phi=\mathrm{ReLU}$, for $l=1,\dots,D$, followed by a head
$z = W_{\mathrm{out}}\phi(h_{D+1})$. Block $l$ has $k=2$ sites: the weight $W_l$ and the activation
$a_l = \phi(h_l)$ entering the matrix multiplication.

\paragraph{Step 1: perturbation of the stream.}
A relative perturbation of size $\eta$ at the activation site changes the branch output by
$\delta u_l = \eta\,\rms(a_l)\,W_l\xi_l$, and at the weight site by
$\delta u_l = \eta\,\rms(W_l)\,\Xi_l a_l$. With variance-preserving weights, $\E\norm{W\xi}^2 \approx
\E\norm{u_l}^2/\rms(a_l)^2$ and $\E\norm{\Xi a}^2\approx \E\norm{u_l}^2/\rms(W_l)^2$ up to $O(1)$
constants, so in both cases $\E\norm{\delta u_l}^2 \approx c_u^2\,\eta^2\,\E\norm{u_l}^2$ for a constant
$c_u$. The stream is perturbed by $s\,\delta u_l$.

\paragraph{Step 2: size of the branch relative to the stream.}
Because $u_l$ and $h_l$ are approximately uncorrelated at a variance-preserving operating point,
$\E\norm{h_{l+1}}^2 \approx (1+s^2 r^2)\,\E\norm{h_l}^2$ with $r^2 = \E\norm{u_l}^2/\E\norm{h_l}^2 = O(1)$.
The relative perturbation injected into the stream at block $l$ therefore has squared size
\[
  \epsilon_l^2 \;=\; \frac{s^2\,\E\norm{\delta u_l}^2}{\E\norm{h_{l+1}}^2}
  \;\approx\; \eta^2\,\frac{c_u^2\, s^2 r^2}{1+s^2r^2}.
\]

\paragraph{Step 3: propagation to the logits.}
By \cref{ass:iso}, a relative stream perturbation of size $\epsilon$ reaches the logits as a relative
perturbation of size $c_h\,\epsilon$, independently of the block. Independence of the $\xi$ across sites
(\cref{prop:sum}) makes the contributions additive:
\[
  G^2 \;=\; G_0^2 \;+\; \sum_{l=1}^{D}\sum_{\text{sites}}\frac{\epsilon_l^2}{\eta^2}\,c_h^2
  \;\approx\; G_0^2 + k\,c^2\,\frac{s^2}{1+s^2r^2}\,D ,
\]
with $c^2 = c_u^2 c_h^2 r^2$. For $s^2 r^2\ll1$ (which holds for $s=1/\sqrt8$ and a fortiori for
$s=1/\sqrt D$) the denominator is $1+O(s^2)$, giving $G^2 = G_0^2 + \kappa s^2 D$ with $\kappa = kc^2$.

\paragraph{Plain networks.} Without the identity path, $h_{l+1} = \phi(W_l h_l)$ and each layer
re-normalises the signal. A relative perturbation at either site of layer $l$ produces a relative
perturbation of order one in $h_{l+1}$, which reaches the logits with factor $c_h$. The same summation
gives $G^2 = G_0^2 + \kappa D$, i.e.\ the residual formula with $s=1$.

\paragraph{Consequences.} With $s = s_0/\sqrt D$ the sum is $\kappa s_0^2$, independent of $D$. With
fixed $s$ it grows linearly in $D$; the local exponent follows by differentiating $\tfrac12\log G^2$ with
respect to $\log D$. \hfill$\square$

\section{Measurement artefacts and how they were corrected}\label{app:artefacts}

We document the artefacts found during the study, because each produced an apparently significant result
that did not survive a better measurement.
\begin{enumerate}[leftmargin=*,itemsep=2pt]
  \item \emph{Noisy amplification estimates.} Early estimates of $G$ used five noise draws. Their
  run-to-run scatter ($\pm10$--$20\%$ per model) is comparable to the entire depth trend of the CNNs
  ($\approx25\%$ between $D=4$ and $64$) and produced two spurious failures of the predictor. All values in
  this paper use 50 draws at three noise levels; the linearity check $G(0.001)/G(0.01)=0.98$--$1.06$
  (architecture means) confirms the small-noise regime.
  \item \emph{Batch-dependent activation scales.} Computing activation quantizer steps from each test
  batch made one image's quantization depend on the other images in its batch and inflated the bit slope
  of plain CNNs ($+0.16$ vs.\ $+0.05$ after calibration). All bit floors reported here use steps
  calibrated once on training data.
  \item \emph{Mean-based step rule and identity-shaped weights.} The LSQ-style rule clips heavily when a
  tensor has a few large entries and many small ones, as in Dirac-initialised kernels (8-bit relative
  error $0.11$--$0.33$ instead of $0.03$). We therefore report no bit floor for that architecture.
  \item \emph{Undertrained deep models.} At 15 epochs the deepest unscaled residual and Dirac CNNs are
  undertrained (full-precision accuracy $0.61$--$0.63$). This inflated both the QAT steepening and the
  bit slope of the unscaled residual CNN; the 30-epoch control (\cref{sec:qat}) quantifies both effects.
  \item \emph{Ratio of exponents.} The ratio $\aqat/\aptq$ is numerically unstable when $\aptq\approx0$
  (a preliminary CNN estimate was $\approx4$). We report the difference $\Delta\alpha$ as the primary
  statistic.
\end{enumerate}

\section{Width}\label{app:width}

\Cref{prop:growth} is stated per layer and says nothing about width. We repeated the MLP measurements at
widths 128 and 512 (five seeds, Fashion-MNIST, logit-norm amplification). The growth law's linear form and
the depth exponents survive; the slope $\kappa$ does not, and falls steeply with width.

\begin{table}[h]
\centering\small
\caption{Width study (plain and fixed-scale residual MLPs; ten seeds at width 32, five at widths 128 and
512). $\kappa$: slope of $G^2$ against $D$ per unit $s^2$ (logit metric); $\rho$: amplification exponent
(logit metric); $\aptq$: collapse exponent. All intervals are 95\%.}
\label{tab:width}
\begin{tabular}{lrccc}
\toprule
Architecture & width & $\kappa$ & $\rho$ & $\alpha_{\mathrm{PTQ}}$ \\
\midrule
plain & 32 & $2.05$ {\scriptsize[1.79, 2.31]} & $0.51\pm0.05$ & $0.70\pm0.07$ \\
plain & 128 & $0.25$ {\scriptsize[0.19, 0.30]} & $0.47\pm0.03$ & $0.53\pm0.01$ \\
plain & 512 & $0.03$ {\scriptsize[0.03, 0.04]} & $0.32\pm0.04$ & $0.53\pm0.18$ \\
residual, $s=1/\sqrt8$ & 32 & $1.68$ {\scriptsize[1.20, 2.17]} & $0.33\pm0.07$ & $0.32\pm0.17$ \\
residual, $s=1/\sqrt8$ & 128 & $0.61$ {\scriptsize[0.53, 0.69]} & $0.42\pm0.05$ & -- \\
residual, $s=1/\sqrt8$ & 512 & $0.15$ {\scriptsize[0.14, 0.17]} & $0.41\pm0.04$ & -- \\
\bottomrule
\end{tabular}

\end{table}

Two observations matter for the main text. First, the agreement between $\aptq$ and $\rho$ for plain MLPs
improves with width ($0.70$ vs.\ $0.51$ at width 32; $0.53$ vs.\ $0.47$ at width 128), so the plain-MLP
exception of \cref{sec:predictor} is largely a narrow-network effect. Second, the value $\kappa\approx2$
that we measured at width 32, and which matched a naive count of two perturbation sites per layer, was a
coincidence of that width: $\kappa$ is an architecture- and width-dependent constant.

\vskip 0.2in
\bibliography{refs}

\end{document}